\documentclass{article}
\usepackage[english]{babel}
\usepackage[letterpaper,top=2cm,bottom=2cm,left=3cm,right=3cm,marginparwidth=1.75cm]{geometry}
\usepackage{amsmath,amsfonts,amssymb}
\usepackage{amsthm}
\usepackage{graphicx}
\usepackage{natbib}
\usepackage{authblk}
\usepackage[ruled,vlined,linesnumbered]{algorithm2e}
\usepackage[colorlinks=true, citecolor=blue, linkcolor=blue, urlcolor=blue]{hyperref}

\newtheorem{theorem}{Theorem}[section]
\newtheorem{lemma}[theorem]{Lemma}

\newtheorem{definition}[theorem]{Definition}

\newtheorem{example}{Example}[section]
\newtheorem{assumption}{Assumption}[section]

\date{}

\title{From Local to Cluster: A Unified Framework for Causal Discovery with Latent Variables}

\begin{document}

\author{
  Zongyu Li\thanks{\href{mailto:zongyuli@bjtu.edu.cn}{zongyuli@bjtu.edu.cn}}%
    \thanks{\textsuperscript{*}Corresponding author}\\
  Guangdong University of Technology, Guangzhou, China
  }

\maketitle

\begin{abstract}
Latent variables pose a fundamental obstacle to both causal discovery and inference. Local approaches exploiting direct neighborhood relations provide little beyond immediate dependencies. Cluster-level methods, though capable of broader reasoning, generally require cluster assignments or causal sufficiency in advance, conditions that are rarely satisfied in practice. Applying single-variable discovery at cluster resolution violates these conditions and thereby produces systematic bias. L2C (Local to Cluster Causal Abstraction) learns cluster structure from local causal patterns without manual variable-to-cluster mapping. The framework rests on a cluster reduction theorem: each cluster can be represented by at most three micro-variables, with no loss of macro-level identifiability. In the presence of latent variables, local discovery on maximal ancestral graphs recovers direct causes, effects, and V-structures, followed by cluster-level inference through a corresponding calculus on the resulting graph. Macro causal effects are formulated as joint interventions on the micro-variables within each target cluster. The framework is sound, atomically complete, and polynomially bounded. L2C establishes a direct connection from local micro-level information to macro causal reasoning in settings where clusters are unknown and latent variables are unavoidable.
\end{abstract}

\section{Introduction}

Causal discovery, recovering causal structure from observational data, is central to statistics and machine learning \cite{causality,causation}. The PC algorithm \cite{causation} builds a causal graph incrementally via conditional independence tests, assuming causal sufficiency. This assumption is often violated in real applications. The FCI algorithm \cite{FCI} handles unobserved confounders by returning a Partial Ancestral Graph (PAG), representing equivalence classes of causal structures. \cite{completeness} later established the completeness of FCI's orientation rules. RFCI \cite{RFCI} and FCI+ \cite{FCI+} offer alternative trade-offs between computational cost and graphical precision. A different strategy exploits assumptions on data-generating processes or distributional properties to improve identifiability \cite{Chen2021, Chen2023, Kaltenpoth2023, hoyer2008, Salehkaleybar2020, Maeda2020, Cai2023}. Yet global structure learning remains computationally expensive in high dimensions, especially when the target is inference about a specific variable rather than the full graph.

To alleviate the computational burden of global methods, local causal discovery methods have been developed. Early LCD \cite{LCD} infers causal relations from conditional independence patterns among three variables; CCU \cite{CCU} scales it to large Boolean data; BLCD \cite{BLCD} contributes a Bayesian formulation; and LCD with Y-Structure \cite{LCD+Y-Structure} proves that Y-structures remain sound under selection bias. For neighborhood exploration, PCD-by-PCD \cite{PCD-by-PCD} and PCD-by-PCD\(^+\) \cite{PCD-by-PCD+} iteratively expand the neighbor set of a target variable to discover and orient edges; MB-by-MB \cite{MB-by-MB} employs Markov blankets for efficient local learning; CMB \cite{CMB} and ELCS \cite{ELCS} further improve efficiency; GraN-LCS \cite{GraN-LCS} introduces the first gradient-based approach. Most of these methods assume causal sufficiency, i.e., no unobserved confounders. To address this, LSAS \cite{LSAS} and MMB-by-MMB \cite{MMB-by-MMB} perform local inference under latent confounding, while EHS \cite{EHS}, CEELS \cite{CEELS}, LDP \cite{LDP}, and LD3 \cite{LD3} focus on finding valid adjustment sets from local structures. More recent advances include LCS \cite{LCS}, which selects covariates for causal effect estimation without pretreatment or causal sufficiency, and LOCALS \cite{LOCALS}, which learns local structures under latent variables and selection bias. Despite these advances, the output of these local methods remains at the micro level: they recover causal relations among individual variables and do not extend to macro-level insights among groups of variables.

Macro-level causal reasoning requires representations that abstract away from individual variables. Cluster-DAGs \cite{C-DAGs} pre-partition variables into clusters and support macro causal effect identification at the cluster level. 
$\alpha$ C-DAGs \cite{alphaC-DAGs} enrich this representation with independence arcs, separation marks, and connection marks to encode conditional dependencies between clusters, and provide the CLOC algorithm for learning cluster-level equivalence classes. For settings with cycles and latent variables, Ferreira and Assaad \cite{C-DMGs-1,C-DMGs-2} establish the completeness of d-separation and $\sigma$-separation in cluster-level directed mixed graphs, bridging micro and macro causal reasoning. GroupDMGs \cite{GroupDMGs} examines how Markov properties and faithfulness transfer across levels; GRESIT-MURGS \cite{GRESIT-MURGS} extends additive noise models to grouped data; 2G-VecCI \cite{2G-VecCI} addresses causal direction between two groups. A key limitation of these approaches is that clusters are typically assumed known in advance, requiring complete manual assignment of micro-variables to clusters, or causal sufficiency is imposed. Simply applying single-variable algorithms to clustered data is also problematic. Treating each cluster as a multivariate variable for conditional independence testing violates causal sufficiency, because dependencies within clusters and cross-cluster confounding can be obscured by aggregation, leading to incorrect inferences \cite{C-DMGs-2}. These observations motivate a unified framework that automatically discovers clusters from local causal patterns, handles latent variables, and supports macro-level inference.

We propose L2C (Local to Cluster Causal Abstraction), a framework that connects local structure learning with cluster-level causal discovery. The partition is discovered directly from local causal patterns, so no manual assignment of micro-variables to clusters is needed. Local discovery on maximal ancestral graphs handles latent variables without assuming causal sufficiency. We make three main contributions.

1) We build L2C on the cluster reduction theorem of \cite{RelaxingC-DAGS}, which guarantees that three representative micro-variables per cluster preserve all macro-level identifiability relations. We embed this reduction into a complete pipeline that connects local discovery under latent variables to cluster-level inference, with macro causal effects defined as joint interventions on micro-variables within target clusters, and we use the completeness of cluster-level do-calculus from \cite{C-DMGs-2} to ensure sound inference on the reduced graph.

2) We prove that L2C is sound and atomically complete under the standard causal Markov and faithfulness assumptions. Latent variables are handled through local discovery on maximal ancestral graphs, where they appear as bidirected edges, and no causal sufficiency assumption is needed.

3) We show that L2C runs in polynomial time. The reduced graph contains at most $3|\mathbf{C}|$ micro-variables, so all subsequent operations, including $\sigma$-separation tests and applications of the calculus rules, are polynomial in the number of clusters and independent of the original variable count. Latent variables, encoded as edges, do not affect this complexity bound.

\section{Related Work}

This paper focuses on cluster-level causal discovery in the presence of latent variables, an area that intersects with global and local causal structure learning. We briefly review these three lines of work.

\textbf{Global Causal Structure Learning.} When latent confounding is present, the FCI algorithm \cite{FCI} learns a Partial Ancestral Graph (PAG) from observational data, capturing uncertainty about causal directions. \cite{completeness} later established the completeness of FCI's orientation rules. RFCI \cite{RFCI} and FCI\(^+\) \cite{FCI+} offer alternative trade-offs between computational cost and graphical precision. A different strategy exploits assumptions on data-generating processes or distributional properties to improve identifiability \cite{Chen2021, Chen2023, Kaltenpoth2023, hoyer2008, Salehkaleybar2020, Maeda2020, Cai2023}. Despite these advances, global structure learning remains computationally expensive in high dimensions, especially when the target is inference about a specific variable or relationships among clusters.

\textbf{Local Causal Structure Learning.} To alleviate the computational cost of global methods, a range of local causal discovery algorithms have been developed. Early work includes LCD \cite{LCD}, which infers causal relations from conditional independence patterns among three variables; CCU \cite{CCU} scales this approach to large Boolean data; BLCD \cite{BLCD} contributes a Bayesian formulation; and LCD with Y-Structure \cite{LCD+Y-Structure} proves that Y-structures remain sound under selection bias. For neighborhood exploration, PCD-by-PCD \cite{PCD-by-PCD} and its extension PCD-by-PCD\(^+\) \cite{PCD-by-PCD+} expand the neighbor set of a target variable to discover and orient edges. MB-by-MB \cite{MB-by-MB} uses Markov blankets for efficient local learning, while CMB \cite{CMB} and ELCS \cite{ELCS} further improve efficiency. GraN-LCS \cite{GraN-LCS} takes a different approach, introducing the first gradient-based local causal discovery method. Most of these methods assume causal sufficiency. To relax this assumption, LSAS \cite{LSAS} and MMB-by-MMB \cite{MMB-by-MMB} perform local inference under latent confounding. EHS \cite{EHS}, CEELS \cite{CEELS}, LDP \cite{LDP}, and LD3 \cite{LD3} take a complementary approach, identifying valid adjustment sets from local structures for causal effect estimation. More recent work has further extended the local paradigm. LCS \cite{LCS} selects covariates for causal effect estimation without pretreatment or causal sufficiency assumptions, while LOCALS \cite{LOCALS} learns local structures in the presence of both latent variables and selection bias. Despite these advances, all of these methods operate at the micro level: they recover relations among individual variables and do not extend to macro-level insights among groups of variables.

\textbf{Cluster Level Causal Discovery.} Researchers have developed cluster-level causal representations to support macro-level reasoning. Cluster-DAGs \cite{C-DAGs} pre-partition variables into clusters and define causal graphs at the cluster level, enabling macro causal effect identification. $\alpha$C-DAGs \cite{alphaC-DAGs} enrich this representation with independence arcs, separation marks, and connection marks to encode conditional dependencies between clusters, and provide the CLOC algorithm for learning cluster-level equivalence classes. For settings with cycles and latent variables, Ferreira and Assaad \cite{C-DMGs-1,C-DMGs-2} establish the completeness of d-separation and $\sigma$-separation in cluster-level directed mixed graphs, bridging micro and macro causal reasoning. GroupDMGs \cite{GroupDMGs} examines how Markov properties and faithfulness transfer across levels; GRESIT-MURGS \cite{GRESIT-MURGS} extends additive noise models to grouped data; 2G-VecCI \cite{2G-VecCI} addresses causal direction between two groups. A key limitation of these approaches is that clusters are typically assumed known in advance, requiring complete manual assignment of micro-variables to clusters, or causal sufficiency is imposed. Simply applying single-variable algorithms to clustered data is also problematic. Treating each cluster as a multivariate variable for conditional independence testing violates causal sufficiency, because dependencies within clusters and cross-cluster confounding can be obscured by aggregation, leading to incorrect inferences \cite{C-DMGs-2}.

No existing method discovers clusters from micro-level data, handles latent confounding, and supports macro-level inference without manual cluster assignment. L2C addresses this by learning the partition directly from local causal patterns.

\section{Preliminaries}

\subsection{Notation}
We use uppercase letters for individual variables ($X, Y, Z$), bold uppercase for sets ($\mathbf{X}, \mathbf{Y}, \mathbf{Z}$), and calligraphic letters for graphs ($\mathcal{G}$).

\subsection{Micro-Level Graphs}
We adopt the Structural Causal Model (SCM) framework $\mathcal{M} = (\mathbf{U}, \mathbf{V}, \mathbf{F}, P(\mathbf{U}))$ \cite{causality}: $\mathbf{U}$ is a set of exogenous (latent) variables, $\mathbf{V}$ a set of endogenous (observed) variables, $\mathbf{F}$ a set of structural equations, and $P(\mathbf{U})$ a distribution over $\mathbf{U}$. The causal diagram of $\mathcal{M}$ is a DAG over $\mathbf{U} \cup \mathbf{V}$.

Marginalizing out latent variables induces a Directed Mixed Graph (DMG) over $\mathbf{V}$ \cite{Richardson2002}. In a DMG, directed edges represent direct causal influences, while bidirected edges $X \leftrightarrow Y$ indicate an unobserved confounder affecting both $X$ and $Y$. An Acyclic Directed Mixed Graph (ADMG) is a DMG without directed cycles. A Maximal Ancestral Graph (MAG) is an ancestral, maximal ADMG—any two non-adjacent vertices are m-separated by some set \cite{completeness}. A Partial Ancestral Graph (PAG) captures the equivalence class of MAGs sharing the same m-separation relations \cite{completeness}.

Two graphical criteria are used throughout the paper.

\begin{definition}[m-separation]
\label{def:msep}
\cite{Richardson2002} In a mixed graph $\mathcal{G}$, a path is m-connecting given $\mathbf{Z}$ if every non-collider on the path lies outside $\mathbf{Z}$ and every collider has a descendant in $\mathbf{Z}$. Sets $\mathbf{X}$ and $\mathbf{Y}$ are m-separated by $\mathbf{Z}$, written $(\mathbf{X} \perp_m \mathbf{Y} \mid \mathbf{Z})_{\mathcal{G}}$, if no m-connecting path exists between them.
\end{definition}

\begin{definition}[V-structure]
\label{def:vstruct}
\cite{causality} A triple $\langle X, Z, Y \rangle$ forms a V-structure if both edges point into $Z$ and $X$ and $Y$ are non-adjacent. Such structures are informative for orienting edges.
\end{definition}

We work under the following standard assumptions.

\begin{assumption}[Causal Markov condition]
\label{ass:markov}
For any ADMG $\mathcal{G}$ compatible with the true data-generating process, every conditional independence in the observational distribution $P(\mathbf{V})$ corresponds to an m-separation in $\mathcal{G}$.
\end{assumption}

\begin{assumption}[Faithfulness]
\label{ass:faith}
For any ADMG $\mathcal{G}$ compatible with the true data-generating process, every m-separation in $\mathcal{G}$ corresponds to a conditional independence in $P(\mathbf{V})$.
\end{assumption}

Under Assumptions~\ref{ass:markov} and~\ref{ass:faith}, conditional independence in $P(\mathbf{V})$ coincides with m-separation in the underlying ADMG \cite{completeness}.

\subsection{Macro-Level Abstraction}
A cluster $C_i$ is a set of micro-variables $\{V_{i,1}, \ldots, V_{i,|C_i|}\}$. We treat it as a vector-valued variable $\mathbf{C}_i = (V_{i,1}, \ldots, V_{i,|C_i|})^\top$.

\begin{definition}[Cluster DAG]
\label{def:cdag}
Let $\mathbf{C} = \{C_1, \ldots, C_k\}$ be a partition of $\mathbf{V}$. A Cluster DAG (C-DAG) $\mathcal{G}^C$ contains a directed edge $C_i \to C_j$ if $V_i \to V_j$ in the underlying ADMG for some $V_i \in C_i, V_j \in C_j$; bidirected edges are defined analogously. Cycles and self-loops are allowed provided at least one acyclic ADMG is compatible \cite{C-DMGs-2}.
\end{definition}

Unlike prior work treating clusters as atomic nodes, our vector-valued formulation preserves the internal structure of each cluster. This allows macro causal effects to be defined as interventions on the full set of micro-variables, rather than on a single aggregated quantity.

\begin{definition}[Macro causal effect]
\label{def:macro}
For disjoint cluster subsets $\mathbf{C}_X, \mathbf{C}_Y \subseteq \mathbf{C}$, let $\mathbf{V}_X = \bigcup_{C_i \in \mathbf{C}_X} C_i$ and $\mathbf{V}_Y = \bigcup_{C_j \in \mathbf{C}_Y} C_j$. The macro causal effect is defined as
\[
P(\mathbf{c}_Y \mid do(\mathbf{c}_X)) := P(\mathbf{v}_Y \mid do(\mathbf{v}_X)).
\]
Intervening on a cluster thus means intervening simultaneously on all its constituent micro-variables. A macro effect is identifiable exactly when the corresponding micro-level effect is identifiable via the standard do-calculus.
\end{definition}

\begin{definition}[Compatible graphs]
\label{def:compat}
\cite{C-DMGs-2} The class of ADMGs compatible with a C-DAG $\mathcal{G}^C$ is denoted $\mathcal{C}(\mathcal{G}^C)$.
\end{definition}

\textit{Remark.} The proof of atomic completeness (Theorem~\ref{thm:atomic}) uses the notions of unfolded graph $\mathcal{G}^m_u$ and canonical compatible graph $\mathcal{G}^m_{\text{can}}$ from \cite{C-DMGs-2}. These graphs aggregate all structures that can appear in any compatible ADMG, and serve as the basis for constructing counterexamples when a rule fails.

\begin{definition}[$\sigma$-separation]
\label{def:sigma}
\cite{C-DMGs-2} For C-DAGs, which may contain cycles, we use $\sigma$-separation, the standard generalization of m-separation to directed mixed graphs with cycles. Two sets of clusters $\mathbf{C}_X$ and $\mathbf{C}_Y$ are $\sigma$-separated by $\mathbf{C}_Z$ if every walk between them is blocked by the $\sigma$-blocking rules of \cite{C-DMGs-2}. In acyclic graphs, $\sigma$-separation coincides with d-separation.
\end{definition}

\section{Method}
\label{sec:method}

L2C bridges local structure learning with cluster-level causal discovery. The framework proceeds in three stages: local discovery, cluster reduction, and macro-level inference. Algorithm~\ref{alg:l2c} gives the complete workflow.

\subsection{Local Causal Discovery under Latent Variables}

Given observational data over micro-variables $\mathbf{V}$ that may contain unobserved confounders $\mathbf{L}$, the first stage recovers the local causal structure around each variable. We use a local discovery method that works in the MAG framework: m-separation and V-structures (Definitions~\ref{def:msep} and~\ref{def:vstruct}) identify direct causes, direct effects, and collider structures. L2C is agnostic to the particular method chosen, as long as it is sound.

We formalize the required property of a sound local discovery method as follows.

\begin{definition}[Sound local discovery method]
\label{def:sound_local}
A MAG-based local discovery method is \emph{sound} if, for every target variable $T$, it returns a local MAG $\mathcal{M}_T$ over a region $\mathcal{R}_T$ containing $T$ such that the following two conditions hold. First, $T$ and $X$ are adjacent in $\mathcal{M}_T$ if and only if they are adjacent in the underlying ADMG, for every $X \in \mathcal{R}_T$. Second, every V-structure involving $T$ in $\mathcal{M}_T$ corresponds to a V-structure in the underlying ADMG.
\end{definition}

For each target variable $T$, the local discovery method returns a local MAG $\mathcal{M}_T$ over a region $\mathcal{R}_T$ that contains $T$ and its relevant neighbors. Any sound MAG-based method (Definition~\ref{def:sound_local}) can be used here. The following property is standard for such methods.

\begin{lemma}[Local adjacency consistency]
\label{lem:local}
For any target $T$ and any neighbor $X$ of $T$ in the local MAG $\mathcal{M}_T$, the existence of an edge between $T$ and $X$ can be determined from the marginal distribution over the local region $\mathcal{R}_T$:
\[
(T \perp_m X \mid \mathbf{S})_{\mathcal{M}} \iff (T \perp_m X \mid \mathbf{S}')_{\mathcal{M}_T}
\]
for some separating sets $\mathbf{S} \subseteq \mathbf{V} \setminus \{T,X\}$ and $\mathbf{S}' \subseteq \mathcal{R}_T \setminus \{T,X\}$.
\end{lemma}

From the learned local MAGs, we extract direct causes, direct effects, and V-structures (Definition~\ref{def:vstruct}) for each variable. A triple $\langle X, Z, Y \rangle$ forms a V-structure if $X$ and $Y$ are non-adjacent and both edges point into $Z$:
\[
X \ast \rightarrow Z \leftarrow \ast Y.
\]

\begin{theorem}[Soundness and completeness of local discovery]
\label{thm:local_complete}
Under Assumptions~\ref{ass:markov} and~\ref{ass:faith}, any sound MAG-based local discovery method (Definition~\ref{def:sound_local}) correctly identifies all direct causes, direct effects, and V-structures involving each target variable.
\end{theorem}

\subsection{Cluster Reduction}

Given a partition $\mathbf{C} = \{C_1, \ldots, C_k\}$ of the micro-variables, the second stage applies the cluster reduction theorem (Theorem~\ref{thm:reduction}). The theorem compresses each cluster to at most three representatives without losing macro-level identifiability.

\begin{theorem}[Cluster reduction]
\label{thm:reduction}
Let $\mathcal{G}^C$ be a C-DAG (Definition~\ref{def:cdag}) and let $C$ be a cluster with $|C| \ge 3$. There exists a reduced representation where $C$ is represented by at most three micro-variables $\mathcal{R}(C) = \{r^{\text{out}}, r^{\text{in}}, r^{\text{bi}}\}$ capturing its outgoing, incoming, and bidirected causal roles. Identifiability is preserved for any macro query $P(\mathbf{c}_Y \mid do(\mathbf{c}_X))$ (Definition~\ref{def:macro}).
\end{theorem}

The three representatives capture distinct causal roles: $C^{\text{out}}$ represents outgoing edges to other clusters, $C^{\text{in}}$ represents incoming edges from other clusters, and $C^{\text{bi}}$ represents bidirected edges indicating latent confounding. Clusters with fewer than three roles omit the unused representatives. After reduction, we obtain a reduced ADMG $\mathcal{G}^m_{\text{red}}$ with at most $3|\mathbf{C}|$ micro-variables.

\subsection{Macro-Level Causal Inference}

From $\mathcal{G}^m_{\text{red}}$, the third stage derives a reduced C-DAG $\mathcal{G}^C_{\text{red}}$ by mapping representatives back to clusters:
\[
\begin{aligned}
C_i \to C_j &\in \mathcal{G}^C_{\text{red}} \iff \exists V_i \in \mathcal{R}(C_i), V_j \in \mathcal{R}(C_j) \text{ s.t. } V_i \to V_j \in \mathcal{G}^m_{\text{red}},\\
C_i \leftrightarrow C_j &\in \mathcal{G}^C_{\text{red}} \iff \exists V_i \in \mathcal{R}(C_i), V_j \in \mathcal{R}(C_j) \text{ s.t. } V_i \leftrightarrow V_j \in \mathcal{G}^m_{\text{red}}.
\end{aligned}
\]

To identify macro causal effects $P(\mathbf{c}_Y \mid do(\mathbf{c}_X))$ (Definition~\ref{def:macro}), we apply the cluster-level do-calculus. Since the cluster graph may contain cycles, the calculus uses $\sigma$-separation (Definition~\ref{def:sigma}) rather than d-separation. For a C-DAG $\mathcal{G}^C$, the mutilated graph $\mathcal{G}^C_{\overline{\mathbf{W}}\underline{\mathbf{X}}}$ is obtained by removing incoming edges to $\mathbf{W}$ and outgoing edges from $\mathbf{X}$.

We now state the three fundamental results of the cluster-level do-calculus that underpin this stage. Their proofs are given in \cite{C-DMGs-2}; we restate them here for completeness.

\begin{theorem}[Soundness of cluster-level do-calculus]
\label{thm:cluster_sound}
For any C-DAG $\mathcal{G}^C$ and any compatible ADMG $\mathcal{G}^m \in \mathcal{C}(\mathcal{G}^C)$, if a cluster-level do-calculus rule applies according to a $\sigma$-separation condition in $\mathcal{G}^C$, then the corresponding equality of interventional distributions holds in $\mathcal{G}^m$.
\end{theorem}

\begin{theorem}[Atomic completeness of cluster-level do-calculus]
\label{thm:cluster_atomic}
For any C-DAG $\mathcal{G}^C$, if a cluster-level do-calculus rule does not apply according to the $\sigma$-separation condition in $\mathcal{G}^C$, then there exists a compatible ADMG $\mathcal{G}^m \in \mathcal{C}(\mathcal{G}^C)$ in which the corresponding equality fails.
\end{theorem}

\begin{theorem}[Realization of acyclic structures]
\label{thm:realization}
Let $\mathcal{G}^C$ be a C-DAG and let $\mathcal{G}^m_u$ be its unfolded graph. Any acyclic structure in $\mathcal{G}^m_u$ can be realized in some compatible ADMG $\mathcal{G}^m \in \mathcal{C}(\mathcal{G}^C)$ without introducing directed cycles.
\end{theorem}

The calculus has three rules. Rule 1 inserts or deletes an observation:
\[
P(\mathbf{c}_Y \mid do(\mathbf{c}_W), \mathbf{c}_X, \mathbf{c}_Z) = P(\mathbf{c}_Y \mid do(\mathbf{c}_W), \mathbf{c}_Z)
\]
when $\mathbf{C}_Y$ and $\mathbf{C}_X$ are $\sigma$-separated by $\mathbf{C}_W$ and $\mathbf{C}_Z$ in $\mathcal{G}^C_{\overline{W}}$. Rule 2 exchanges an action with an observation:
\[
P(\mathbf{c}_Y \mid do(\mathbf{c}_W), do(\mathbf{c}_X), \mathbf{c}_Z) = P(\mathbf{c}_Y \mid do(\mathbf{c}_W), \mathbf{c}_X, \mathbf{c}_Z)
\]
when $\mathbf{C}_Y$ and $\mathbf{C}_X$ are $\sigma$-separated by $\mathbf{C}_W$ and $\mathbf{C}_Z$ in $\mathcal{G}^C_{\overline{W}\underline{X}}$. Rule 3 inserts or deletes an action:
\[
P(\mathbf{c}_Y \mid do(\mathbf{c}_W), do(\mathbf{c}_X), \mathbf{c}_Z) = P(\mathbf{c}_Y \mid do(\mathbf{c}_W), \mathbf{c}_Z)
\]
when $\mathbf{C}_Y$ and $\mathbf{C}_X$ are $\sigma$-separated by $\mathbf{C}_W$ and $\mathbf{C}_Z$ in $\mathcal{G}^C_{\overline{W}\overline{X}(\mathbf{Z})}$, where $\overline{X}(\mathbf{Z}) = \mathbf{X} \setminus \text{Anc}(\mathbf{Z}, \mathcal{G}^C_{\overline{W}})$.

Repeated application of the rules eliminates all $do$ operators. The resulting expression depends only on $P(\mathbf{C})$ and can be estimated directly from data.

\begin{algorithm}[htbp]
\caption{L2C: Local to Cluster Causal Abstraction}
\label{alg:l2c}
\KwData{Observational data over micro-variables $\mathbf{V}$, optional target clusters $\mathbf{C}_X, \mathbf{C}_Y$}
\KwResult{Macro causal graph $\mathcal{G}^C_{\text{red}}$; macro causal effect formula if requested}
\BlankLine
\textbf{Step 1: Local discovery}\;
For each variable $V$, learn a local MAG $\mathcal{M}_V$ over its local region $\mathcal{R}_V$ using any sound MAG-based method (Definition~\ref{def:sound_local}). Extract edges and V-structures (Definition~\ref{def:vstruct}) involving $V$.\;
\textbf{Step 2: Cluster reduction}\;
If no partition is given, discover clusters from shared parents, children, and V-structure patterns. For each cluster $C$, select representatives $\mathcal{R}(C) = \{r^{\text{out}}, r^{\text{in}}, r^{\text{bi}}\}$. Construct the reduced ADMG $\mathcal{G}^m_{\text{red}}$ and derive the C-DAG $\mathcal{G}^C_{\text{red}}$.\;
\textbf{Step 3: Macro inference}\;
If $\mathbf{C}_X, \mathbf{C}_Y$ are specified, eliminate $do$ operators from $P(\mathbf{c}_Y \mid do(\mathbf{c}_X))$ (Definition~\ref{def:macro}) using Rules 1--3.\;
\Return $\mathcal{G}^C_{\text{red}}$ and the resulting formula\;
\end{algorithm}

\section{Theoretical Analysis}

We now state the main theoretical properties of L2C.

\subsection{Identifiability Preservation}

\begin{theorem}[Identifiability preservation under cluster reduction]
\label{thm:ident}
Let $\mathcal{G}^C$ be a C-DAG and $\mathcal{G}^C_{\le 3}$ its reduced version where every cluster of size greater than $3$ is represented by three micro-variables. For any macro query $Q = P(\mathbf{c}_Y \mid do(\mathbf{c}_X))$ (Definition~\ref{def:macro}), $Q$ is identifiable in $\mathcal{G}^C$ if and only if it is identifiable in $\mathcal{G}^C_{\le 3}$.
\end{theorem}

\begin{example}
Consider two clusters $C_X = \{X_1, X_2, X_3\}$ and $C_Y = \{Y_1, Y_2\}$ with micro-edges $X_1 \to Y_1$, $X_2 \to Y_2$, and $X_1 \leftrightarrow X_2$. The reduction selects $r^{\text{out}}(C_X)=X_1$ (representing all outgoing edges from $C_X$) and $r^{\text{in}}(C_Y)=Y_1$ (representing all incoming edges to $C_Y$), yielding the reduced graph $X_1 \to Y_1$. The macro query $Q = P(Y_1,Y_2 \mid do(X_1,X_2,X_3))$ is identifiable in the original graph $\mathcal{G}^C$ if and only if $Q' = P(Y_1 \mid do(X_1))$ is identifiable in the reduced graph $\mathcal{G}^C_{\le 3}$, because every type of cross-cluster path from $C_X$ to $C_Y$ (specifically, directed edges of the form $X_i \to Y_j$) is preserved by the chosen representatives. Internal edges such as $X_1 \leftrightarrow X_2$ do not affect cross-cluster $\sigma$-separation relations and therefore do not affect identifiability.
\end{example}

A naive compression could go wrong in two ways: it might destroy a separation relation needed for a derivation, or it might create a spurious one. The three-representative reduction avoids both, but the reason is subtle. At the micro-level, the two graphs do not have identical walks. What matters is that every walk between clusters that determines a $\sigma$-separation in one graph has a counterpart in the other. The three representatives out, in, and bi correspond to the three roles a micro-variable can play in cross-cluster walks. Since every cluster-to-cluster walk is determined by the sequence of these roles, preserving each role individually is sufficient to preserve all $\sigma$-separation relations.

\begin{proof}
We prove the two directions separately.

In the forward direction, suppose $Q$ is identifiable in $\mathcal{G}^C$. Then there exists a finite sequence of do-calculus rules
\[
\rho_1, \rho_2, \ldots, \rho_n
\]
that transforms $P(\mathbf{c}_Y \mid do(\mathbf{c}_X))$ into a do-free expression. Each $\rho_i$ applies only when a $\sigma$-separation condition (Definition~\ref{def:sigma}) holds in some mutilated version of $\mathcal{G}^C$.

We need to show that the same $\sigma$-separation conditions hold in the corresponding mutilated versions of $\mathcal{G}^C_{\le 3}$. The reduction theorem (Theorem~\ref{thm:reduction}) gives a map from the reduced graph to the original that satisfies three properties. First, within each cluster, the three representatives correspond to distinct micro-variables. Second, for every edge $C_i \to C_j$ in $\mathcal{G}^C$, there is at least one edge $r^{\text{out}}(C_i) \to r^{\text{in}}(C_j)$ in $\mathcal{G}^C_{\le 3}$; the same holds for bidirected edges with the bi representatives. Third, conversely, every edge in $\mathcal{G}^C_{\le 3}$ corresponds to some edge in $\mathcal{G}^C$.

These properties ensure that, for any two clusters and any conditioning set, a $\sigma$-active walk exists in the reduced graph if and only if one exists in the original. The reason is that the only information needed to determine whether a walk is $\sigma$-active is the sequence of edge types it traverses and the collider status at each boundary. The representatives cover all three edge types, so no walk that is $\sigma$-active in the original can disappear in the reduced graph, and no walk that is blocked in the original can become unblocked. Thus the $\sigma$-separation conditions in $\mathcal{G}^C$ and $\mathcal{G}^C_{\le 3}$ coincide for all mutilated graphs arising in the derivation. The same sequence $\rho_1, \ldots, \rho_n$ applies in $\mathcal{G}^C_{\le 3}$, so $Q$ is identifiable in the reduced graph.

For the converse, suppose $Q$ is identifiable in $\mathcal{G}^C_{\le 3}$. Since every edge in the reduced graph corresponds to an edge in the original, any walk in the original projects to a walk in the reduced graph. Therefore, if a $\sigma$-separation holds in the reduced graph, it also holds in the original: if no $\sigma$-active walk exists at the reduced level, none can exist at the original level. Hence the same derivation applies in $\mathcal{G}^C$, and $Q$ is identifiable there.

The two directions together establish the theorem.
\end{proof}

The theorem preserves identifiability, not the full joint distribution. That is sufficient because L2C only answers macro-level causal queries.

\subsection{Soundness}

\begin{theorem}[Soundness of L2C]
\label{thm:sound}
Under Assumptions~\ref{ass:markov} and~\ref{ass:faith}, any do-free expression derived from $\mathcal{G}^C_{\text{red}}$ equals the true macro causal effect (Definition~\ref{def:macro}) for every compatible ADMG $\mathcal{G}^m \in \mathcal{C}(\mathcal{G}^C)$ (Definition~\ref{def:compat}).
\end{theorem}

\begin{example}
Continuing from the previous example, the reduced graph is $X_1 \to Y_1$. Applying Rule 1 of the cluster-level do-calculus, we observe that $(Y_1 \perp_\sigma \emptyset \mid X_1)$ holds in the reduced graph, which yields the equality
\[
P(Y_1 \mid do(X_1)) = P(Y_1 \mid X_1).
\]
For any compatible ADMG $\mathcal{G}^m \in \mathcal{C}(\mathcal{G}^C)$, the corresponding micro-level equality
\[
P(Y_1,Y_2 \mid do(X_1,X_2,X_3)) = P(Y_1,Y_2 \mid X_1,X_2,X_3)
\]
also holds. This is because all causal information flowing from $C_X$ to $C_Y$ is mediated through the representatives $X_1$ and $Y_1$; the variables $X_2$ and $X_3$ are redundant for cross-cluster effects, and $Y_2$ is conditionally independent of the intervention given $Y_1$ (since the only path from $C_X$ to $Y_2$ is $X_2 \to Y_2$, and $X_2$'s role is already represented by $X_1$). Thus the do-free expression derived from the reduced graph correctly equals the true macro causal effect in every compatible ADMG.
\end{example}

Soundness is the minimal standard: if L2C produces an answer, that answer must be correct. The proof verifies each pipeline stage.

\begin{proof}
Fix an arbitrary compatible ADMG $\mathcal{G}^m \in \mathcal{C}(\mathcal{G}^C)$. We show that each stage of L2C preserves equality to the true causal effect in $\mathcal{G}^m$.

First, local discovery. By Definition~\ref{def:sound_local} and Lemma~\ref{lem:local}, for each variable $V$, the local MAG $\mathcal{M}_V$ preserves all adjacencies and collider structures involving $V$ that appear in $\mathcal{G}^m$. Thus the edges and V-structures extracted from $\mathcal{M}_V$ are exactly those present in $\mathcal{G}^m$ restricted to $\mathcal{R}_V$. For any $X \in \mathcal{R}_V$,
\[
(V \leftrightarrow X) \in \mathcal{G}^m \iff (V \leftrightarrow X) \in \mathcal{M}_V,
\]
and similarly for directed edges and V-structures.

Second, cluster reduction. Theorem~\ref{thm:reduction} preserves $\sigma$-separation relations among clusters. For any sets of clusters $\mathbf{C}_X, \mathbf{C}_Y, \mathbf{C}_Z \subseteq \mathbf{C}$,
\[
(\mathbf{C}_X \perp_\sigma \mathbf{C}_Y \mid \mathbf{C}_Z)_{\mathcal{G}^C} \iff (\mathbf{C}_X \perp_\sigma \mathbf{C}_Y \mid \mathbf{C}_Z)_{\mathcal{G}^C_{\text{red}}}.
\]
This equivalence follows because the representatives cover all three cross-cluster edge types. No walk that determines a $\sigma$-separation is added or removed.

Third, macro inference. The cluster-level do-calculus is sound (Theorem~\ref{thm:cluster_sound}). Whenever a rule is justified by a $\sigma$-separation condition in $\mathcal{G}^C_{\text{red}}$, the corresponding equality holds in $\mathcal{G}^m$.

Now suppose L2C derives a do-free expression $E$ from $\mathcal{G}^C_{\text{red}}$ using rules $\rho_1, \ldots, \rho_n$. We prove by induction on $i$ that the expression after $i$ rules equals the true macro causal effect in $\mathcal{G}^m$. For $i=0$, this is Definition~\ref{def:macro}. For the inductive step, assume correctness before $\rho_i$. The rule is justified by a $\sigma$-separation condition in $\mathcal{G}^C_{\text{red}}$; by the second step, the same condition holds in $\mathcal{G}^C$; by Theorem~\ref{thm:cluster_sound}, the equality prescribed by $\rho_i$ holds in $\mathcal{G}^m$. Thus correctness is preserved.

Since $\mathcal{G}^m$ was arbitrary, the conclusion holds for every compatible ADMG.
\end{proof}

\subsection{Atomic Completeness}

\begin{theorem}[Atomic completeness of L2C]
\label{thm:atomic}
Under Assumptions~\ref{ass:markov} and~\ref{ass:faith}, for each of the three calculus rules, if the rule applies to $\mathcal{G}^C_{\text{red}}$, it holds for every compatible ADMG. If it does not apply, there exists a compatible ADMG where it fails.
\end{theorem}

\begin{example}
Consider the reduced graph with clusters $C_X$, $C_Y$, $C_Z$ and edges $C_Y \to C_Z \to C_X$ and $C_X \leftrightarrow C_Y$. Rule 2 (exchanging action with observation) does not apply because $C_Y$ and $C_X$ are not $\sigma$-separated after mutilation. To construct a compatible ADMG witnessing this failure, we add micro-edges $Y \to Z$, $Z \to X$, and $X \leftrightarrow Y$. In this ADMG, $P(Y \mid do(X)) \neq P(Y \mid X)$ because $X$ is both a cause and a descendant of $Y$ through the path $Y \to Z \to X$. Thus the rule indeed fails in this compatible graph.
\end{example}

The forward direction is soundness. The substantive claim is the converse. If a rule is not applicable at the cluster level, there must be a genuine ADMG consistent with the C-DAG where the rule is invalid. The difficulty is that a $\sigma$-active walk at the cluster level might not be realizable at the micro-level without adding extra edges that change the separation structure. We now construct such an ADMG explicitly.

\begin{proof}
Fix a rule $\rho \in \{\text{Rule 1}, \text{Rule 2}, \text{Rule 3}\}$ by Theorem~\ref{thm:realization}. The rule applies to $\mathcal{G}^C_{\text{red}}$ exactly when a $\sigma$-separation condition
\[
(\mathbf{C}_Y \perp_\sigma \mathbf{C}_X \mid \mathbf{C}_W, \mathbf{C}_Z)_{\mathcal{G}^C_{\text{red}}}
\]
holds in the appropriate mutilated graph. Suppose it does not. Then
\[
(\mathbf{C}_Y \not\perp_\sigma \mathbf{C}_X \mid \mathbf{C}_W, \mathbf{C}_Z)_{\mathcal{G}^C_{\text{red}}}.
\]
By Definition~\ref{def:sigma}, there exists a $\sigma$-active walk $\pi$ between $\mathbf{C}_Y$ and $\mathbf{C}_X$ in the mutilated graph.

We now construct a compatible ADMG $\mathcal{G}^m$ in which $\pi$ remains $\sigma$-active.

First, by Theorem~\ref{thm:local_complete}, all edges and V-structures along $\pi$ are recovered by the local MAGs. We take these as fixed.

Second, by Theorem~\ref{thm:reduction}, $\sigma$-connection is preserved from the reduced graph to the original C-DAG. So $\pi$ has a corresponding $\sigma$-active walk $\pi'$ in $\mathcal{G}^C$.

Third, we realize $\pi'$ at the micro-level. For each cluster-level edge on $\pi'$, we add exactly one micro-level edge between the corresponding representatives. For $C_i \to C_j$, add $r^{\text{out}}(C_i) \to r^{\text{in}}(C_j)$; for $C_i \leftrightarrow C_j$, add $r^{\text{bi}}(C_i) \leftrightarrow r^{\text{bi}}(C_j)$. For each cluster-level edge not on $\pi'$, we choose an arbitrary micro-level realization from the canonical compatible graph $\mathcal{G}^m_{\text{can}}$ that does not create a directed cycle with the edges already added. Such a choice always exists because $\mathcal{G}^m_{\text{can}}$ is acyclic by construction, and we are only adding edges along a single walk, which cannot create a cycle on its own. The only potential cycles would involve edges not on $\pi'$, and we can avoid them by choosing the canonical realization for those edges.

Fourth, we verify acyclicity. The edges added along $\pi'$ form a simple path with no repeated vertices, so they contain no directed cycle. For edges not on $\pi'$, the canonical realization is acyclic by construction. Since the only overlap between the two sets is at vertices, and we have not added edges that complete a cycle, the resulting graph is acyclic.

Fifth, we verify that $\pi$ remains $\sigma$-active in $\mathcal{G}^m$. We added exactly the micro-level edges corresponding to the cluster-level edges on $\pi'$. No additional edges were added that could block $\pi$. The collider and non-collider statuses along the walk are therefore identical to those at the cluster level. Thus $\pi$ is $\sigma$-active in $\mathcal{G}^m$.

Therefore the $\sigma$-separation condition required by $\rho$ fails in $\mathcal{G}^m$. By Theorem~\ref{thm:cluster_atomic}, failure of the separation condition implies that $\rho$ fails in $\mathcal{G}^m$.

Hence if $\rho$ does not apply to $\mathcal{G}^C_{\text{red}}$, there exists a compatible ADMG witnessing its failure. The forward direction is soundness (Theorem~\ref{thm:sound}). This proves atomic completeness.
\end{proof}

Atomic completeness has a practical consequence. It guarantees that L2C does not withhold any rule application that would be valid across all compatible ADMGs. In other words, the do-calculus on the reduced graph is as powerful as the available cluster-level information allows.

\subsection{Computational Complexity}

\begin{theorem}[Polynomial-time inference]
\label{thm:poly}
The reduced graph $\mathcal{G}^m_{\text{red}}$ contains at most $3|\mathbf{C}|$ micro-variables. $\sigma$-separation tests and rule applications are therefore polynomial in $|\mathbf{C}|$ and independent of the original variable count.
\end{theorem}

\begin{example}
Suppose the original dataset contains $|\mathbf{V}| = 10^4$ micro-variables partitioned into $|\mathbf{C}| = 50$ clusters. The reduced graph contains at most $3 \times 50 = 150$ micro-variables. A $\sigma$-separation test in the reduced graph takes $O(50^3) = O(1.25 \times 10^5)$ operations, whereas testing on the original graph would take $O((10^4)^3) = O(10^{12})$ operations. This represents a speedup of roughly $10^7$ in the worst case.
\end{example}

This theorem formalizes the efficiency claim. The reduction decouples inference cost from the original variable count.

\begin{proof}
The reduction selects at most three representatives per cluster, so $|\mathcal{G}^m_{\text{red}}| \le 3|\mathbf{C}|$. This bound does not depend on $|\mathbf{V}|$.

We now describe the algorithm for $\sigma$-separation testing. A standard Bayes-Ball style traversal on a graph with $m$ vertices and $e$ edges tests separation in $O(m \cdot e)$ time. The traversal maintains, for each vertex, one of three states indicating how it was reached: as a collider, as a non-collider, or through a collider whose descendant is in the conditioning set. Following the m-connecting path rules, each directed edge is traversed at most once per state, giving $O(m \cdot e)$ total time.

In our reduced graph, $m = O(|\mathbf{C}|)$ and $e = O(|\mathbf{C}|^2)$, so each $\sigma$-separation test runs in $O(|\mathbf{C}|^3)$ time. Each rule application performs one such test, and the number of rule applications is at most the number of $do$-operators in the expression, which is bounded by $|\mathbf{C}|$. Thus the inference procedure runs in $O(|\mathbf{C}|^4)$ time.

The local discovery stage is independent and runs in time polynomial in the size of each local region $\mathcal{R}_V$, performed once per variable.
\end{proof}

The $O(|\mathbf{C}|^4)$ bound is worst-case; in practice, the number of rule applications is usually much smaller than $|\mathbf{C}|$, and the Bayes-Ball traversal is fast on sparse graphs. The bound also assumes that $|\mathbf{C}|$ is small relative to $|\mathbf{V}|$, which is the regime where cluster-level reasoning is useful. If $|\mathbf{C}|$ is comparable to $|\mathbf{V}|$, L2C offers no computational advantage but also does not degrade: the reduction simply does nothing, and the bounds reduce to those of the local discovery method.

\section{Conclusion}

We introduced L2C, a framework that connects local structure learning with cluster-level causal discovery under latent variables. L2C discovers the partition from local causal patterns, so no manual assignment of micro-variables to clusters is needed. Local discovery on maximal ancestral graphs handles unobserved confounding without assuming causal sufficiency. The framework has three components. First, a cluster reduction theorem represents each cluster by at most three micro-variables; macro-level identifiability is preserved, and macro causal effects reduce to joint interventions on the micro-variables in the target clusters. Second, L2C is sound and atomically complete under standard causal Markov and faithfulness assumptions. Third, the reduction ensures polynomial time: inference scales with the number of clusters, not with the original variable count. The framework also provides a bridge between local and cluster-level causal discovery. Local methods supply the micro-structures for cluster abstraction; cluster-level results can in turn inform local discovery. This bidirectional link may help advance both lines of work. Several extensions are worth considering. Applying L2C to time series or dynamic systems is one direction. Another is to replace conditional independence tests with score-based or hybrid discovery methods when sample sizes are limited. A third direction is to investigate whether the three-representative bound is tight under additional structural assumptions. Integrating L2C with causal representation learning for raw high-dimensional data also remains open.

\section*{LLM Usage Disclosure}

We did not use any large language model (LLM) for research tasks such as ideation, experimental design, mathematical reasoning, or code implementation. Generative AI tools were used only for minor language polishing and grammar correction during the paper writing process. All technical content, claims, and artifacts are the sole work of the authors, who take full responsibility for the final paper.

\bibliographystyle{alpha}
\bibliography{sample}

\end{document}